\documentclass{article}

\usepackage{amsmath}
\usepackage{mathrsfs}
\usepackage{mathtools}
\usepackage{amssymb}
\usepackage{enumerate}
\usepackage{dsfont}
\usepackage{bm}
\usepackage{graphicx}
\usepackage{hyperref}
\usepackage{physics}
\usepackage[a4paper, lmargin=2.5cm, rmargin=2.5cm, bmargin=4cm, marginpar=3.0cm]{geometry}
\usepackage{natbib}
\usepackage{authblk}

\usepackage{xcolor}
\usepackage[normalem]{ulem}

\numberwithin{equation}{section}

\usepackage{amsthm}

\newtheorem{theorem}{\textbf{Theorem}}

\newtheorem{definition}[theorem]{\textit{Definition}}

\newtheorem{corollary}[theorem]{\textbf{Corollary}}
\newtheorem{lemma}[theorem]{\textbf{Lemma}}
\newtheorem*{remark}{\textbf{Remark}}

\DeclarePairedDelimiterX\set[1]\lbrace\rbrace{\def\given{\;\delimsize\vert\;\allowbreak}#1}

\newcommand{\B}{\mathcal{B}}
\newcommand{\D}{\mathcal{D}}
\newcommand{\E}{\mathbb{E}}
\newcommand{\F}{\mathcal{F}}

\renewcommand{\H}{\mathcal{H}}

\newcommand{\N}{\mathcal{N}}
\newcommand{\M}{\mathcal{M}}
\newcommand{\R}{\mathbb{R}}
\newcommand{\V}{\mathcal{V}}
\newcommand{\W}{\mathcal{W}}
\newcommand{\X}{\mathcal{X}}

\newcommand{\di}{\diamond}

\DeclareMathOperator*{\relu}{ReLU}
\DeclareMathOperator*{\spn}{span}

\usepackage{parskip}

\title{\vspace{-1.5cm}Deep Networks are Reproducing Kernel Chains}
\author[1,$\ddag$,*]{Tjeerd Jan Heeringa}
\author[1,*]{Len Spek}
\author[1]{Christoph Brune}
\affil[1]{Mathematics of Imaging \& AI, University of Twente, Enschede, The Netherlands}
\affil[$\ddag$]{Corresponding author: t.j.heeringa@utwente.nl}
\affil[*]{Authors contributed equally}
\date{\today}

\begin{document}

\maketitle

\begin{abstract}
    Identifying an appropriate function space for deep neural networks remains a key open question. While shallow neural networks are naturally associated with Reproducing Kernel Banach Spaces (RKBS), deep networks present unique challenges. In this work, we extend RKBS to chain RKBS (cRKBS), a new framework that composes kernels rather than functions, preserving the desirable properties of RKBS. We prove that any deep neural network function is a neural cRKBS function, and conversely, any neural cRKBS function defined on a finite dataset corresponds to a deep neural network. This approach provides a sparse solution to the empirical risk minimization problem, requiring no more than $N$ neurons per layer, where $N$ is the number of data points.
    \\ \\
    \textbf{keywords: }Neural networks, Reproducing Kernel Banach Spaces, Representer Theorem 
\end{abstract}

\section{Introduction}
While deep neural networks have proven very powerful for many machine learning problems, a fundamental understanding of such methods is still being developed. A key open question is which function space is appropriate for deep neural networks. 

For shallow neural networks, appropriate spaces are Reproducing Kernel Banach Spaces (RKBS). A well-known example is the Barron space \citep{e_towards_2020,spek_duality_2023}. Such RKBS share many of the properties of the widely successful Reproducing Kernel Hilbert Spaces (RKHS). Desirable properties include their reproducing properties through their kernel, their sparsity through the representer theorem and their concise description \citep{bartolucci_understanding_2023}. An appropriate space for deep networks should preferably have these properties as well, and be a Banach space for all commonly used activation functions.
\begin{figure}[h!]
    \centering
    \includegraphics[width=0.65\linewidth]{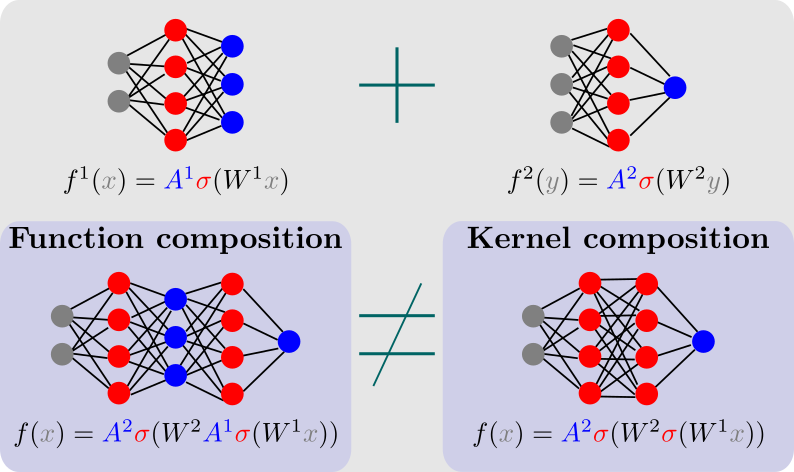}
    \caption{Deep networks are not compositions of shallow networks. (left) Function composition leads to \textit{undesired} extra bottleneck layer \textcolor{blue}{$A^1$ in blue}; (right) Kernel composition effectively \textit{matches} layers directly.}
    \label{fig:compositions}
\end{figure}

The key observation is that deep networks are not compositions of shallow networks. As seen in Figure~\ref{fig:compositions}, matching one network's output with another network's input introduces an extra linear layer, which is not present in a conventional deep network construction. Instead, via kernel composition the hidden layers are matched directly.

We use this intuition to construct our RKBS for deep networks, where instead of composing the functions of shallow RKBS, we only compose their kernels. This allows all the desirable properties of the shallow neural network spaces to be carried over and some of these to be strengthened.

In this paper, we introduce a constructive method for composing kernels of RKBSs, which we call kernel chaining. A special subclass called neural chain RKBS (cRKBS) represents neural networks: Any deep neural network is a neural cRKBS function, and any neural cRKBS function over a finite number of data points is a deep neural network. We show that these chain spaces satisfy a representer theorem, in which the sparse solution has in each layer at most as many neurons as the number of data points. 

\subsection{Related work}
Deep networks and their corresponding spaces can be grouped into three categories. Before discussing these, we introduce foundational related work about shallow networks.

\subsubsection{Shallow neural networks}
Reproducing Kernel Hilbert Spaces (RKHS) are Hilbert spaces with the extra requirement that the point evaluations are bounded linear functionals. These point evaluations are linked to a kernel function through the Riesz-Fréchet representation theorem. This allows function evaluations to be written as an inner product of the kernel and the function itself. There is a unique link between the RKHS and its kernel in the sense that every RKHS has a unique kernel and every kernel uniquely defines a RKHS \citep{aronszajn_theory_1950}. An important result in RKHS theory is the representer theorem. This states that the regularized minimization problem of finding the best fitting function in the RKHS given some data points has a sparse solution. This turns the infinite-dimensional minimization problem over the whole RKHS into a finite-dimensional one with dimension bounded by the number of samples, which can then be solved using a standard least squares approach.

Shallow neural networks correspond to RKHS in both the random feature limit and the lazy training limit. In the random feature limit, the hidden-layer parameters are subsampled from a particular probability distribution $\pi$, and then the output layer takes a linear combination \citep{rahimi_random_2007}. If we however consider infinite linear coefficients $a\in L^2(\pi)$, we get functions $f$ of an RKHS 
\begin{equation}
    f(x) = \int a(w,b) \sigma(w^Tx + b) d\pi(w,b)
\end{equation}
The kernel of this RKHS depends on both the nonlinearity in the network and the particular parameter distribution. This interpretation has strong connections with Gaussian processes, and thus this limit is also called the GP-limit \citep{hanin_random_2023}. In the lazy training limit, the key observation is that the hidden-layer parameters don't change much during training for sufficiently large networks and randomised initializations. Hence, the networks can be approximated by a linearization around such an initialization. These linearizations are elements of a RKHS with a kernel depending on both the gradient and the random initialization. This kernel is called the Neural Tangent Kernel (NTK), and thus this limit is also called the NTK-limit \citep{jacot_neural_2020}.

Neither the random feature limit nor the lazy training limit captures the key aspects of neural networks \citep{chizat_lazy_2019, woodworth_kernel_2020}. In particular, a RKHS seems to be too small to have the right adaptability seen in neural networks \citep{bach_breaking_2017}. Four major research directions can be seen as extensions to the RKHS theory. These are the extended RKHS (eRKHS), variational splines, Reproducing Kernel Banach Spaces (RKBS) and the Barron spaces. 

The extended RKHS is our term for the asymmetric kernel RKHS and the Hyper-RKHS \citep{he_learning_2022, he_learning_2024, liu_generalization_2021}. The main idea is to split the problem into two pieces: First, you determine a suitable kernel from a potentially infinite set of kernels for your RKHS and then determine the right function from the RKHS corresponding to this kernel. The resulting function spaces are still RKHS.

For the variational splines, the main idea is that ReLU functions are examples of splines \citep{parhi_banach_2021}. Using classical spline theory, a function space for neural networks can be constructed that has a similar representer theorem. A benefit of this approach is that the norm can be defined without making references to the weights of the network. A downside of this approach is that it does not work for all activation functions.

Instead of a fixed distribution $\pi$ to sample the weights, the Barron space considers all possible probability distributions $\pi$ that satisfy a growth limit \citep{e_towards_2020}
\begin{equation}
    f(x) = \int a\sigma(\braket{x}{w}+b)d\pi(a,w,b)
\end{equation}
Here, the growth limit is imposed to ensure the integral is well-posed when $\sigma$ is Lipschitz. The formulation implies that Barron space is a Banach space isomorphic to a quotient over growth-limited Radon measures with bounded point evaluations \citep{e_representation_2022}. The space is not an RKHS, but a union of infinitely many RKHS \citep{e_barron_2022, spek_duality_2023}. On compact domains, Barron spaces embed into $C^{0,1}$ and $L^p$ spaces \citep{e_representation_2022,e_barron_2022} and functions in the space can be approximated by finite width neural networks in $L^2$ and $L^\infty$ with bound scaling proportionally to the inverse square root of the width of the network \citep{e_towards_2020}. Moreover, their duality structure is well-understood \citep{spek_duality_2023} and their dependence on the activation function characterized \citep{heeringa_embeddings_2024,li_complexity_2020,caragea_neural_2020}. They satisfy a representer theorem \citep{parhi_banach_2021, bartolucci_understanding_2023}, but the finite-dimensional problem is now a non-convex optimisation problem and can't be solved by a least-squares approach like RKHS. For ReLU, the corresponding Barron space agrees with the variational spline function space \citep{bartolucci_understanding_2023}.

RKBSs are the Banach analogue to RKHSs in the sense that they are Banach spaces of functions with bounded point evaluations \citep{zhang_reproducing_2009,lin_reproducing_2022}. When RKBSs are not Hilbert spaces, they don't have access to the Riesz-Fréchet representation theorem for Hilbert spaces. This breaks the symmetry between primal and dual but allows for larger and more expressive function spaces. RKBSs always come in RKBS pairs: an RKBS $\B$ over $X$ and an RKBS $\B^\di$ over $\Omega$ with a non-degenerate pairing between them. The kernel is a scalar function of $X$ and $\Omega$. For an RKHS $\H$, $\H=\B=\B^\di$ and $X=\Omega$. The archetypical example of a non-Hilbertian RKBS is the space of continuous functions on a compact domain with the supremum norm. The Barron spaces (and thus also the variational splines) are instances of neural RKBSs \citep{spek_duality_2023}, a subclass of the integral RKBSs. Results for Barron spaces carry over to the neural RKBS, and several have extensions to the vector-valued case \citep{bartolucci_understanding_2023, parhi_banach_2021, shenouda_variation_2024, spek_duality_2023} and some have extensions to integral RKBSs considering different pairings than the measures with the (growth-limited) continuous functions, like ones based on Lizorkin distributions or Kantorovich-Rubinstein norms \citep{neumayer_explicit_2023, bartolucci_lipschitz_2024}. 

\subsubsection{Deep neural networks}
As mentioned, the approaches for deep networks can be grouped into three categories: the generalized Barron spaces, the hierarchical spaces and the bottlenecked spaces. These three categories are extensions to the directions introduced under shallow neural networks before.

The generalized Barron spaces category consists of the neural tree spaces \cite{e_banach_2020}, a direct extension of the Barron spaces. These spaces are layered, with each layer defined based on the previous one, branching out in a tree-like fashion. The first layer is a standard Barron space. The next layer is constructed by integrating over the unit ball of functions of the previous layer, i.e. the functions are of the form 
\begin{equation}
    f(x) = \int \sigma(g(x))d\mu(g)
\end{equation}
with $g$ a function from the neural tree spaces with one less depth. The neural tree spaces consist of Lipschitz functions, have bounded point evaluations, and satisfy a direct approximation theorem as well as an inverse approximation theorem. Thereby extending several of their previously proven results from Barron spaces over to the neural tree spaces. The representer theorem shows that a sparse solution with at most $N^\ell$ neurons for layer $\ell$ exists, for a total of at most $N^{2L-1}$ parameters.

The hierarchical spaces category consists of the Neural Hilbert Ladder (NHL) \citep{chen_neural_2024}, a direct extension of RKHS with hierarchical kernels \citep{huang_hierarchical_2023}. They have now multiple RKHS arranged in a layerwise fashion, each with multiple different choices of kernel. These follow a similar approach as the neural tree spaces, in the sense that the first kernel is given by
\begin{equation}
    k(x_1,x_2) = \braket{x_1}{x_2}_\X
\end{equation}
and the later kernels are given
\begin{equation}\label{eq:NHL_kernel}
    k^{\ell+1}(x_1,x_2) = \int\sigma(g(x_1))\sigma(g(x_2))d\pi^\ell(g)
\end{equation}
with the integral over the functions from the NHL with one less depth. To define the kernel in \eqref{eq:NHL_kernel}, a choice has to be made for the probability measure $\pi^\ell$. Each sequence of $\set{\pi^\ell}_{\ell=1}^L$ defines a hierarchical RKHS. The $(L,p)$-NHL functions space $\F^{(L)}_p$ consists of functions with a finite complexity,
\begin{equation}\label{eq:NHL_quasi-norm}
    \mathfrak{C}^{(L)}_p(f) := \inf_{\H} \norm{f}_{\H}\D^{(L)}_p(\H)
\end{equation}
with the Hilbert spaces $\H$ in the infimum being constructed using a sequence of probability measures $\set{\pi^\ell}_{\ell=1}^L$ and $\D^{(L)}_p(\H)$ a weight for that space. The $\mathfrak{C}^{(L)}_p$ is only a quasi-norm when $\sigma$ is homogeneous due to its otherwise unbounded quasi-triangle constant. For $L=p=2$, the NHL $\F^{(L)}_p$ is equivalent to a Barron space since the complexity agrees then with the Barron norm. 

The bottle-necked spaces consist of all approaches in which shallow neural networks are composed function-wise. The non-Hilbertian approaches within this category are the deep RKBSs and the deep variational splines \citep{bartolucci_neural_2024, parhi_what_2022}. In the former, they are compositions of vector-valued variational spline spaces and in the latter, the functions are compositions of vector-valued RKBSs. In both cases, the resulting set of functions is not a normed vector space and is dependent on the intermediate sets. The deep variational splines are restricted to using the rectifier power unit (RePU), the higher order version of ReLU, as an activation function, whereas the deep RKBSs have no such restriction. The networks in the deep variational splines have skip connections, whereas the networks in the deep RKBSs do not. In the representer theorem, the former has neural networks with alternating finite and uncountable layers, whereas the latter has networks with alternating countable and uncountable layers. The difference in construction results in different scaling in the number of parameters in the representer theorem, with the former having at most $\sum_{\ell=2}^{L}d_\ell(Nd_{\ell}+d_{\ell-1}+N)$ total parameters for a network with predetermined intermediate sizes $\set{d^\ell}_{\ell=1}^{L}$ and the latter having at most $\sum_{\ell=2}^{L}d_\ell(d_{\ell-1}+1)$ total parameters for a network for some set of intermediate sizes $\set{d^\ell}_{\ell=1}^{L}$ satisfying $d^{\ell+1}\leq Nd^{\ell}$.

The generalized Barron spaces are limited to ReLU, their kernel structure is unknown and their representer theorem has a solution growing exponentially in width with increasing depth. The hierarchical spaces are, just like the bottle-necked spaces, no longer normed function spaces. Hence, neither of these three categories provides a function space satisfying the necessary and preferred properties for being an appropriate function space for deep neural networks. Therefore, the key question is still open. 

\subsection{Our contribution}
In this work, we introduce a constructive method for composing kernels of RKBSs which results in chain RKBSs (cRKBS). These spaces are distinct from deep compositional RKBS of \citet{bartolucci_neural_2024} in the sense that not the functions of the respective spaces are composed but the kernels of the respective spaces. We show that cRKBS preserve the RKBS structure with a proper kernel. 

Next, we make the cRKBS framework more concrete by focussing on integral RKBS, where functions are defined as integrals over either the first or second argument of the kernel with respect to some measure. In particular, we focus on kernels which are a combination of an elementwise non-linearity and an affine transformation. We call integral RKBS with such kernels neural RKBS. We develop a concise formula of the function of neural chain RKBS in terms of the functions of a neural cRKBS with one less layer.

This leads to the following main theorem of this work, which precise statement is given by Theorems~\ref{thm:deep_networks_in_neural_cRKBS}~and~\ref{thm:finite_data_neural_cRKBS}. 

\begin{theorem}\label{thm:summary}
Every deep neural network of depth $L$ is an element of the neural cRKBS for depth $L$. Conversely, if we only consider $N$ data points, then all functions in a neural cRKBS are deep networks with at most $N$ hidden nodes per layer and all the weights, except the last layer, are shared.
\end{theorem}

We prove this theorem by leveraging the primal-dual relation between RKBS pairs and that the activation function acts only elementwise. This also implies that neural cRKBSs satisfy a representer theorem with sparse solutions of most $N$ hidden neurons for each layer where $N$ equal to the amount of data points, for a total of at most $N(N+1)(L+1)$ parameters. 

\section{Chain Reproducing Kernel Banach Spaces}
In this section, we will start by reviewing the theory of Reproducing Kernel Banach Spaces. Afterwards, we will introduce a procedure to construct chain RKBS by composing their kernels. This is done by iteratively adding 'links' to make a 'chain'.

\subsection{Reproducing Kernel Banach Spaces}
We start with a Banach space of functions $\B$ on a domain $X$ mapping to the reals. This means that elements of $\B$ are determined by function evaluation on $X$, i.e. $f(x)=0$ for all $x\in X$ implies that $f$ is the zero vector. This also means that $\B$ contains true functions, not function classes like in a Lebesgue space. We say that $\B$ is reproducing on $X$ when point evaluation is a bounded functional.
\begin{definition}\label{def:rkbs}
Let $\B$ be a Banach space of real functions on a domain $X$. $\B$ is a Reproducing Kernel Banach space (RKBS) if, for every $x\in X$, there exists a $C_x>0$
\begin{equation}
|f(x)|\leq C_x \|f\|_{\B}
\end{equation}
for all $f\in \B$.    
\end{definition} 
Notably, this constant $C_x$ is independent of $f$. This means that the functional which evaluates functions at $x$, denoted by $K_x$, must be an element of the dual space $\B^*$. 

The archetypical example of such an RKBS is the Banach space of continuous functions $C(X)$ on a compact set $X$, where the max-norm trivially satisfies the above condition. The corresponding functional here is the point measure $\delta_x$, which is an element of the dual of the continuous functions, the space of Radon measures $\M(X)$.

A common way of constructing different RKBSs is with a feature space $\Psi$ and a feature map $\psi:X\mapsto \Psi^\ast$.
\begin{theorem}\citep[Proposition 3.3]{bartolucci_understanding_2023}\label{thm:rkbs_quotient}
A Banach space $\B$ of functions on $X$ is reproducing if and only if there exists a Banach space $\Psi$ and a map $\psi: X \mapsto \Psi^*$ such that 
\begin{equation}
\begin{split}
    \B \simeq \Psi/\mathcal{N}(A)\\
    \|f\|_{\B} = \inf_{f=A\nu} \|\nu\|_{\Psi}
\end{split}
\end{equation}
where the linear transformation $A$ maps elements of the feature space $\Psi$ to functions on $X$ and is defined as
\begin{equation}
    (A\nu)(x) := \langle \psi(x)| \nu \rangle
\end{equation}
for all $x\in X$ and $\nu \in \Psi$.
\end{theorem} 
Here, the feature map $\psi$ effectively selects which elements of the dual space are defined to be the point evaluation functionals $K_x$. 

Similar to an RKHS, we can define a reproducing kernel for an RKBS. However, the lack of an inner product structure complicates a straightforward generalisation of such a kernel. Instead, one needs to carefully consider the dual structure of the Banach space. In this work, we will work with a dual pair of Banach spaces to avoid technical difficulties with non-reflexive Banach spaces in further chapters.
\begin{definition}
A dual pair of Banach spaces $\B$, $\B^{\di}$ is defined as a pair of Banach spaces together with continuous bilinear map (pairing) $\langle \cdot|\cdot \rangle: \B^\di\times \B \to \R$ with the bound 
\begin{equation}\label{eq:pairing_bound}
|\langle g,f \rangle| \leq \|f\|_{\B} \|g\|_{\B^\di}
\end{equation}
for all $f\in \B$, $g \in \B^\di$ and such that the pairing is non-degenerate, i.e.
\begin{equation}\label{eq:pairing_no_deg}
\begin{split}
    &\langle g,f \rangle =0 \quad \forall f\in \B\,\,  \implies g=0 \\
    &\langle g,f \rangle = 0 \quad\forall g\in \B^\di \implies f=0
\end{split}
\end{equation}
\end{definition}

If both of these spaces are an RKBS, each with their own domain, we can define a reproducing kernel on these domains, if they contain the others' evaluation functionals.

\begin{definition}\label{def:rkbs2}
Let $\B$, $\B^{\di}$ be a dual pair of Banach spaces with the pairing $\langle \cdot|\cdot \rangle$. Let $\B$ be an RKBS with domain $X$ and $\B^\di$ an RKBS with domain $\Omega$. 

If there exists a function $K:X \times \Omega \mapsto \mathbb{R}$ such that $K(x,\cdot)\in \B^{\di}$ for all $x\in X$ and $K(\cdot, w)\in \B$ for all $w\in \Omega$ and 
\begin{equation}\label{eq:kernel}
\begin{split}
    f(x) &= \langle K(x,\cdot)| f\rangle\\
    g(w) &= \langle g| K(\cdot, w)\rangle 
\end{split}
\end{equation}
for all $f\in \B, g\in \B^\di$ and $x\in X, w\in \Omega$, then we call $K$ the reproducing kernel of the RKBS pair $\B$, $\B^{\di}$.
\end{definition}

Compared to RKHS, this reproducing kernel $K$ is not symmetric or positive definite but has related properties. If we interchange both $\B$ and $\B^\di$, and $X$ and $\Omega$, the function $K^*(w,x)=K(x,w)$ is the reproducing kernel of the interchanged pair, which gives a related notion of symmetry. Note that in the Hilbert setting $B\cong B^\di$ and $X=\Omega$, which makes $K$ a symmetric function. 

As an inner product is positive definite, an RKHS kernel must be as well. In the RKBS case, we only have a bounded bi-linear pairing, so the kernel needs only to be independently bounded, i.e. there exist functions $C_X:X\to \R, C_\Omega:\Omega \to \R$, such that $|K(x,w)|\leq C_X(x) C_\Omega(w)$ for all $x\in X, w\in \Omega$.

We can formulate a similar theorem as \citep{aronszajn_theory_1950}, that each symmetric, positive definite kernel defines an RKHS. For an RKBS pair the kernel is unique and independently-bounded, and every independently-bounded kernel defines a pair of RKBSs.

\begin{theorem}\label{thm:kernel}
Let $\B$, $\B^{\di}$ be an RKBS pair. The reproducing kernel $K$ is unique, independently bounded and given by $K(x,w)= \langle K_x | K_w\rangle$ for all $x\in X, w\in \Omega$. 

Conversely, given some sets $X,\Omega$, and an independently bounded kernel $K: X\times \Omega \to \R$, there exists an RKBS pair $\B$, $\B^{\di}$ with domains $X$ and $\Omega$ respectively, and $K$ as the reproducing kernel.
\end{theorem}
\begin{proof}
Let $\B$, $\B^{\di}$ be an RKBS pair. By Definition~\ref{def:rkbs2} we get that $K_x=K(x,\cdot)\in \B^\di$ and $K_w=K(\cdot, w)\in \B$ for all $x\in X, w\in \Omega$. It follows that we can apply \eqref{eq:kernel} to $f=K_w$ to get that
\begin{equation}
    K(x,w)=K_w(x) = \braket{K(x,\cdot)}{K_w} = \braket{K_x}{K_w}
\end{equation}
Hence, the $K$ is uniquely defined by the evaluation functionals $K_x, K_w$. Finally, the boundedness of the pairing implies that $|K(x,w)|\leq \|K_x\|_{\B} \|K_w\|_{\B^\di}$.

Conversely, let $X,\Omega$ be sets and let $K: X\times \Omega \to \R$ be independently bounded. We define the vector spaces of real functions 
\begin{equation}
    \begin{split}
        \mathcal{V}&:= \text{span}\set{K(\cdot,w): X\to \R \given w\in \Omega}\\
        \mathcal{V}^\di&:= \text{span}\set{K(x,\cdot): \Omega \to \R \given x\in X}
    \end{split}
\end{equation}
and define the bi-linear map between them as
\begin{equation}
    \braket{\sum_j a_j K(x_j,\cdot)}{\sum_i c_i K(\cdot,w_i)} := \sum_{i,j}a_jc_i K(x_j,w_i)
\end{equation}

Choose as norm for $\mathcal{V}$ the weighted supremum norm
\begin{equation}
    \norm{f}_{\mathcal{V}} := \sup_{x\in X}\frac{\abs{\braket{K(x,\cdot)}{f}}}{C_X(x)}
\end{equation}
With this norm $K(x,\cdot)$ becomes a bounded linear functional, i.e. for each $x\in X$ 
\begin{equation}
    \abs{\braket{K(x,\cdot)}{f}} \leq C_X(x)\norm{f}_{\mathcal{V}}
\end{equation}
for all $f\in \mathcal{V}$. Let $\B$ be the completion of $\mathcal{V}$ with respect to its norm, i.e.
\begin{equation}
   \B := \overline{\mathcal{V}}^{\norm{\cdot}_{\mathcal{V}}}
\end{equation}
It is immediate that $f(x)=0$ for all $x\in X$ implies that $f=0$. Hence, this is an RKBS with domain $X$.

Define now the norm on $\mathcal{V}^\di$ as 
\begin{equation}
    \norm{g}_{\mathcal{V}^\di} := \sup_{\norm{f}_\B\leq 1}\abs{\braket{g}{f}}
\end{equation}
where the pairing here is the extension of the pairing between $\V,\V^\di$ to $\B,\V^\di$. With this norm, the evaluation functionals $K(\cdot,\omega)$ are bounded linear functionals on $\mathcal{V}^\di$, since
\begin{equation}
    \abs{\braket{g}{K(\cdot,\omega)}} \leq \norm{g}_{\mathcal{V}^\di}\norm{K(\cdot,\omega)}_\B < \infty
\end{equation}
holds for all $g\in \mathcal{V}^\di$ and $\omega\in \Omega$. Similarly to $\B$, we define $\B^\di$ as the completion of $\mathcal{V}^\di$ with respect to its norm, i.e.
\begin{equation}
    \B^\di := \overline{\mathcal{V}^\di}^{\norm{\cdot}_{\mathcal{V}^\di}}
\end{equation}
By construction, this is a Banach space of functions with domain $\Omega$. 

By Hahn-Banach, the pairing for $\mathcal{V},\mathcal{V}^\di$ can be extended to a pairing for $\B,\B^\di$ which satisfies \eqref{eq:pairing_bound}. Moreover, if $g\in \B^\di$, then 
\begin{equation}
    \forall f\in \B\backslash\set{0}:\;\braket{g}{f} = 0 \implies \forall \omega\in\Omega:\; g(\omega) = \braket{g}{K(\cdot,\omega)} =0 \implies g=0
\end{equation}
which follows from $\B^\di$ being a Banach space of functions. The proof for $f\in\B$ follows similarly. Combined these show that the pairing is non-degenerate, and that, therefore, $\B,\B^\di$ is an RKBS pair.
\end{proof}

Note that the constructed RKBS pair $\B$, $\B^\di$ is not unique, due to the choice of the norm on $\B$. We could instead of the supremum norm have chosen a $p$-type norm for $\B$ and a $q$-type norm for $\B^\di$, where $\tfrac{1}{p}+\tfrac{1}{q}=1$.

\subsection{Constructing a chain Reproducing Kernel Banach Space}\label{sec:chaining_rkbs}
We start with an RKBS pair $\B^1, \B^{1\di}$ with domains $X$ and $\Omega^1$ respectively, and kernel $K^1$, which we call the \textit{initial RKBS pair}. To construct a \textit{chain RKBS pair}, the RKBS pair $\B^2$, $\B^{2\di}$ with domains $X$ and $\Omega^2$, we need to first introduce a different pair of RKBS, which we call a \textit{link RKBS pair}, which we denote with a tilde.

Let $\tilde{\B}^2$, $\tilde{\B}^{2\di}$ be an RKBS pair with domains $\B^{1\di}$ and $\Omega^2$, and has a kernel $\tilde{K}^2$. Note that $\tilde{\B}^2$ contains functions of $\B^{1\di}$, which itself are functions of $\Omega^1$. The kernels reproduce the functions in the following way
\begin{equation}
\begin{split}
    f^1(x) &= \braket{K^1_x}{f^1}  \qquad \,\, g^1(w^1) = \braket{g^1}{K^1_{w^1}}  \\
    h^2(g^1) &= \braket{\tilde{K}^2_{g^1}}{h^2} \qquad q^2(w^2) = \braket{q^2}{\tilde{K}^2_{w^2}}
\end{split}
\end{equation}
for all $f^1\in \B^1, g^1 \in \B^{1\di}, h^2 \in \tilde{\B}^2, q^2\in \tilde{\B}^{2\di}$, and $x\in X, w^1 \in \Omega^1, w^2 \in \Omega^2$. Note that if $\tilde{\B}^{2}$ contains only linear functions, the space is basically just $\B^1$ and the rest of the construction becomes trivial. Therefore, in later sections, we choose $\tilde{\B}^{2}$ such that it contains nonlinear functions, which is relevant when we apply this to neural networks. Also, for neural networks the set $\Omega^2$ is chosen as $\B^1$, which makes the link spaces more symmetrical, but this is not necessary for the general construction.

By composing the kernels in a certain way, we create a new RKBS $\B^2$ with domain $X$. To be precise, we define $\B^2$ using the quotient space of Theorem~\ref{thm:rkbs_quotient} with the feature map $\psi:X\rightarrow \tilde{\B}^{2\di}$:
\begin{equation}\label{eq:feature_map}
    \psi(x):= \tilde{K}^2_{K^1_x}
\end{equation}
One way to view this construction is that the kernel $K^1$, selects which functionals of $\tilde{B}^2$ become the evaluation functionals of $\B^2$. The feature map is well-defined, because $K^1_x$ is an element of $\B^{1\di}$, by definition of the adjoint pair of RKBS. This defines the linear map $A$, which maps elements $h\in \tilde{\B}^2$ to functions of $X$
\begin{equation}
    (Ah)(x):= \braket{\psi(x)}{h}
\end{equation}

We formally define the Banach space $\B^2$ as the image of the map $A$, which is a subspace of the vector space of functions with domain $X$. The norm of $\B^2$ is defined by the norm of the quotient space $\tilde{\B}^2$ over the nullspace $\N(A)$.
\begin{equation}
\begin{split}
\B^2 &:= A(\tilde{\B}^2)\simeq\tilde{\B}^2/\N(A)\\
\|f^2\|_{\B^2} &:= \inf_{Ah=f^2} \|h\|_{\tilde{\B}^2}\\
f^2(x) &:= (Ah)(x) = \braket{\tilde{K}^2_{K^1_x}}{h} = h(K^1_x)
\end{split}
\end{equation}
for all $f^2\in \B^2$, $x\in X$ and for any $h\in \tilde\B^1$ such that $Ah=f^2$. The nullspace of $A$ is given by
\begin{equation}\label{eq:nullspace_A}
    \N(A)=\set{h\in \tilde{\B}^2 \given h(K^1_x)=0 \quad \forall x\in X}
\end{equation}

We define the space $\B^{2\di}$ as the subspace of $\tilde{\B}^{2\di}$ which annihilates $\mathcal{N}(A)$, i.e.
\begin{equation}
    \B^{2\di} := \set{q\in \tilde{\B}^{2\di} \given \braket{q}{h}=0 \quad\forall h\in \mathcal{N}(A)}
\end{equation}
The subspace embedding is given by the identity map $A^\di: \B^{2\di} \to \tilde{\B}^{2\di}$, and the norm as well as function evaluation on $\B^{2\di}$ are inherited from $\tilde{\B}^{2\di}$. 
\begin{equation}
    \begin{split}
        \|g^2\|_{\B^{2\di}} &:= \|A^\di g^2\|_{\tilde{\B}^{2\di}}\\
        g^2(w^2) &:= (A^{\di} g^2)(w^2)
    \end{split}
\end{equation}
for all $g^2\in \B^{2\di}$ and $w^2\in \Omega^2$. 

The pairing between $\B^2$ and $\B^{2\di}$ is also inherited
\begin{equation}\label{eq:chain_pairing}
    \braket{g^2}{f^2} := \braket{A^\di g^2}{h^2}
\end{equation}
for all $f^2\in \B^2, g^2\in \B^{2\di}$ and some $h^2\in \tilde{\B}^2$ such that $f^2=Ah^2$.

Now that we have defined the chain spaces $\B^2$ and $\B^{2\di}$, we show that this construction indeed generates a new RKBS pair.

\begin{theorem}[RKBS Consistency]\label{thm:chain_well-posed}
The spaces $\B^2$, $\B^{2\di}$ form an RKBS pair with kernel $K^2(x,w) := \tilde{K}^2(K^1_x,w)$ for all $x\in X, w\in \Omega^2$.   
\end{theorem}
\begin{proof}
By Theorem~\ref{thm:rkbs_quotient}, we immediately get that $\B^2$ is an RKBS. Since $\tilde{\B}^{2\di}$ is an RKBS and $\B^{2\di}$ is equipped with the subspace norm, $\B^{2\di}$ must also be an RKBS by Definition~\ref{def:rkbs}.

The pairing \eqref{eq:chain_pairing} is bounded, as
\begin{equation}
    \braket{g}{f}| = |\braket{A^\di g}{h}| \leq \|A^\di g\|_{\tilde{\B}^{2\di}}\|h\|_{\tilde{\B}} = \|g\|_{\B^{2\di}}\|h\|_{\tilde{\B}^2}
\end{equation}
for all $f\in \B^2,g\in \B^{2\di}$ and all $h\in \tilde{\B}^2$ such that $f=Ah$. Taking the infimum over all such $h$, we get the required bound.

Next, to show that the pairing is non-degenerate: First, if $f\in \B^2$ such that $\braket{g}{f}=0$ for all $g\in \B^{2\di}$, then $\braket{A^\di g}{h}=0$ for all $g\in \B^{2\di}$ and $h\in \tilde{\B}^2$ such that $f=Ah$. From the definition of $\B^{2\di}$ it follows that $h\in \N(A)$, and thus $f=A h = 0$. Second, if $g\in \B^{2\di}$ such that $\braket{g}{f} = 0$ for all $f\in \B^2$, then $\braket{A^\di g}{h}=0$ for all $h\in \tilde{\B}^2$. Hence, by the non-degeneracy of this pairing, we get that $A^\di g=0$ and as the embedding is injective, $g=0$. Thus, we conclude that $\B^2,\B^{2\di}$ form a dual pair of Banach spaces.

Let $x\in X$. For the evaluation functional $K^2_x := K^2(x, \cdot)$, we get for all $h\in \mathcal{N}(A)$ that $\braket{\tilde{K}^2(K^1_x,\cdot)}{h^2} = h^2(K^1_x)=0$, as $\tilde{K}^2$ is a kernel. Hence $K^2_x\in \B^{2\di}$ and 
\begin{equation}
    \braket{K^2_x}{f} = \braket{A^\di K^2_x}{h} = \braket{\tilde{K}^2(K^1_x,\cdot)}{h} = f(x)
\end{equation}
for all $f\in \B^2$ and all $h\in\tilde{\B}^2$ such that $f=Ah$.

Let $w\in \Omega^2$. For the evaluation functional $K^2_w := K^2(\cdot, w)= A(\tilde{K}^2(\cdot,w))$. Hence, $K^2_w\in \B^2$ and 
\begin{equation}
    \braket{g}{K^2_w} = \braket{A^\di g}{\tilde{K}^2(\cdot,w)} = (A^\di g)(w)= g(w)
\end{equation}
for all $g\in \B^{2\di}$.

Thus, we conclude that $\B^2$, $\B^{2\di}$ form a pair of RKBS with kernel $K^2$.
\end{proof}

By defining new link RKBS pairs $\tilde{\B}^\ell$, $\tilde{\B}^{\ell \di}$ for $\ell = 2,\hdots, L$, we can build longer chain RKBS pairs $\B^L, \B^{L \di}$ for any length $L\in \mathbb{N}$. Note that this procedure is not symmetric: $\B^L$ still has domain $X$, while $\B^{L \di}$ has domain $\Omega^L$. This makes sense from a standpoint of neural networks as if we add layers, we get more weights, while the input space stays the same.

For Hilbert spaces, the construction takes the following simplified form. We start with an initial RKHS $\H^1$ with domain $X$ and kernel $K^1:X\times X\to \R$, and a link RKHS $\tilde{\H}^2$ with domain $\H^1$ with kernel $\tilde{K}^2: \H^1\times\H^1 \to \R$. The procedure above then leads to a chain RKHS $\H^2$ with domain $X$ and kernel $K^2(x,x')= \tilde{K}^2(K^1(x,\cdot),K^1(\cdot,x'))$.

\begin{figure}
    \centering
    \includegraphics[width=0.9\linewidth]{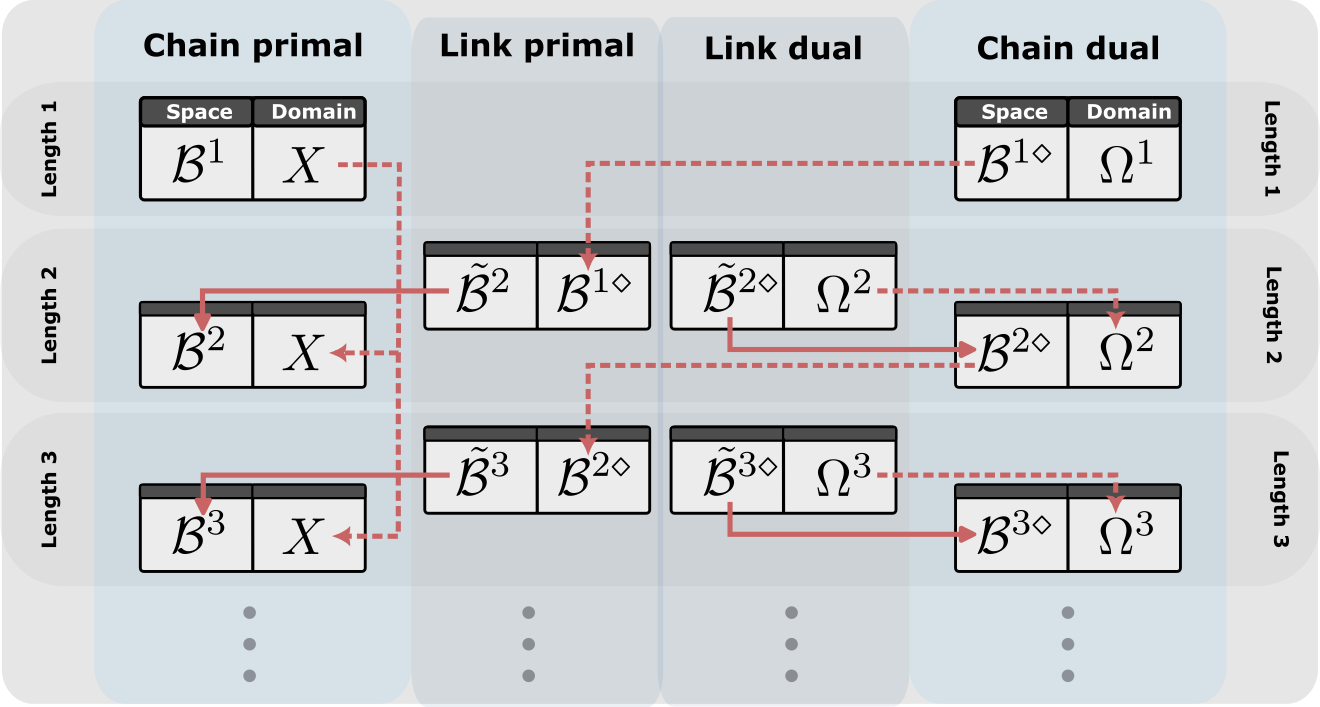}
    \caption{Schematic representation of the construction of a kernel chain. The dependencies of the spaces have been indicated by arrows: a straight line when there is a map between the chain and link, and a dashed line for a domain identification.}
    \label{fig:kernel_chain}
\end{figure}

\section{Chain RKBS for Deep Neural Networks}
In the previous section, we showed that given a kernel, there exists a RKBS pair with that kernel, which we can use to build a chain RKBS. In this section, we will make this construction more concrete by choosing specific spaces such that it correspond to deep neural networks. 

First, we restrict ourselves to integral chain RKBS, or integral cRKBS for short. Afterwards, we will choose the kernel such that it corresponds to an element-wise activation function and an affine combination, which is common in neural networks. We will show that for the ReLU activation function, our description is equivalent to the generalised Barron spaces of \citet{e_banach_2020}. 

\subsection{Integral cRKBS}
Integral RKBSs are RKBSs where the functions with domain $X$ are described by integrating the kernel with respect to some measure over $\Omega$ (or vice versa).  

To be precise, we start with completely regular Hausdorff sets $X$ and $\Omega$ and a bounded, measurable function $\varphi:X\times \Omega \to \R$, where measurable is always understood to be Borel measurable. We define $\M(X),\M(\Omega)$ to be the Banach spaces of Radon measures of $X$, $\Omega$ respectively with the total variation norm, i.e. regular signed Borel measures with finite total variation. The Dirac or point measures are denoted by $\delta_x \in \M(X), \delta_w\in \M(\Omega)$, for $x\in X$ and $w\in \Omega$ respectively. We then define the \textit{integral RKBS} pair as follows:

\begin{definition}\label{def:int_type}
Let $X,\Omega$ be sets with a completely regular Hausdorff topology, and let $\varphi:X\times \Omega \to \R$ be bounded and measurable. An integral RKBS pair $\B, \B^{\di}$ is defined as follows: For all $f\in\B$ there exists a $\mu\in \M(\Omega)$ such that
\begin{equation}\label{eq:int_type-primal}
\begin{split}
    f(x) &:= (A_{\Omega\to X}\mu)(x) := \int_\Omega \varphi(x,w) d\mu(w)\\
    \|f\|_\B &:= \inf_{f=A_{\Omega\to X}\mu} |\mu|(\Omega)
\end{split}
\end{equation}
for all $x\in X$, for all $g\in \B^\di$ there exists a $\rho\in \M(X)$ such that 
\begin{equation}\label{eq:int_type-dual}
\begin{split}
    g(w) &:= (A_{X\to \Omega}\rho)(w) := \int_X \varphi(x,w) d\rho(x)\\
    \|g\|_{\B^\di} &:= \sup_{w\in\Omega} |g(w)|
\end{split}
\end{equation}
for all $w\in \Omega$, and the pairing between $f\in\B$ and $g\in\B^\di$ is given by
\begin{equation}\label{eq:int_type-pairing}
    \braket{g}{f} := \int_{X\times \Omega} \varphi(x,w) d(\rho \times \mu)(x,w) 
\end{equation}
for some $\mu\in\M(\Omega),\rho\in\M(X)$ such that $f=A_{\Omega\to X}\mu$ and $g=A_{X\to \Omega}\rho$.
\end{definition}
Note that $\B$ is always defined to have a '1-norm' and $\B^\di$ an '$\infty$-norm'. In a previous paper \citep{spek_duality_2023}, we have already shown that $\B,\B^\di$ satisfy all the conditions to be an RKBS pair, if $X\subseteq \R^d, \Omega \subseteq \R^n$ and the kernel $\varphi$ is continuous. The proof for this more general setting is analogous, so we will only give a short sketch of the proof.
\begin{theorem}\label{thm:int_type}
The spaces $\B$, $\B^\di$ with the pairing of Definition~\ref{def:int_type} form a pair of RKBS with kernel $\varphi$.
\end{theorem}
\begin{proof}
The space $\B$ is an RKBS by Theorem~\ref{thm:rkbs_quotient} with $\psi(x)=\varphi(x,\cdot)$ and $\Psi=\M(\Omega)$. The space $\B^\di$ is an RKBS directly by Definition~\ref{def:rkbs}. The pairing is well-defined, bounded and non-degenerate due to Fubini's Theorem \citep[Lemma 14]{spek_duality_2023}. The evaluation functionals are given by $K_x = A_{X\to \Omega}\delta_x \in \B^\di, K_w = A_{\Omega\to X}\delta_w \in \B$, and the kernel $K(x,w) = \braket{K_x}{K_w} =  \varphi(x,w)$. 
\end{proof}

Now, to expand our initial RKBS pair $\B,\B^\di$ to an integral cRKBS pair, we can use the procedure of the Section~\ref{sec:chaining_rkbs}. We only need to choose completely regular Hausdorff $\Omega^\ell$ and bounded, measurable kernels $\tilde{\varphi}^\ell: \B^{(\ell-1)\di} \times \Omega^\ell \to \R$ for some $\ell=2,\cdots, L$. We can then define the link RKBS pairs $\tilde{\B}^\ell, \tilde{\B}^{\ell\di}$ to be the integral RKBS pair corresponding to $\tilde{\varphi}$, see Definition~\ref{def:int_type}. Note that here elements of $\tilde{\B}^\ell$ are functions of $\B^{(\ell-1)\di}$, i.e. for each $h\in \tilde{\B}^\ell$ there exists a $\mu\in \M(\Omega^\ell)$
\begin{equation}
    h(g) := (A_{\Omega^\ell \to \B^{(\ell-1)\di}}\mu)(g) := \int_{\Omega^\ell} \tilde{\varphi}(g,w)d\mu(w)
\end{equation}
for all $g\in \B^{(\ell-1)\di}$. 

From these link RKBS pairs $\tilde{\B}^\ell, \tilde{\B}^{\ell\di}$, we can use the construction of Section~\ref{sec:chaining_rkbs} to formulate a cRKBS pair $\B^\ell, \B^{\ell\di}$ with domains $X,\Omega^\ell$. These spaces are again integral RKBS with $\varphi^\ell(x,w) = \tilde{\varphi}^\ell(\varphi(x,\cdot), w)$

\begin{theorem}[Integral Consistency]\label{thm:int_type_chain}
Let $L\in \mathbb{N}$, $L\geq 2$. If $\B^1, \B^{1\di}$ is an integral RKBS pair with domains $X,\Omega^1$ and kernel $\varphi^1$, and $\tilde{\B}^\ell, \tilde{\B}^{\ell\di}$ are integral RKBS pairs with domains $\B^{(\ell-1)\di},\Omega^\ell$ with kernel $\tilde{\varphi}^\ell$ for all $\ell=2,\hdots,L$, then the cRKBS pair $\B^L, \B^{L\di}$ with domains $X,\Omega^L$ is an integral RKBS pair with kernel $\varphi^L(x,w) = \tilde{\varphi}^L(\varphi^{L-1}(x,\cdot), w)$.
\end{theorem}
\begin{proof}
We will prove that $\B^L, \B^{L\di}$ are integral RKBS pairs by induction on $L$ and then the form of the kernel $\varphi^L(x,w)$ follows directly from Theorem~\ref{thm:chain_well-posed}. The base case of $L=1$ is true by definition. 

Suppose $\B^{L-1}, \B^{(L-1)\di}$ are an integral cRKBS pair, for some $L\geq 2$. Then by Theorem~\ref{thm:chain_well-posed} $\B^L,\B^{L\di}$ form an RKBS pair with domains $X,\Omega^L$ and kernel $\varphi^L= \tilde{\varphi}^L(\varphi^{L-1}(x,\cdot), w)$. As $\tilde{\varphi}^L, \varphi^{L-1}$ are bounded, so is $\varphi^L$, and by Theorem~\ref{thm:kernel} this kernel is unique. What remains to be shown is that $\B^L,\B^{L\di}$ satisfy the extra requirements to be of integral RKBS, as in Definition~\ref{def:int_type}.

To check \eqref{eq:int_type-primal}, define 
\begin{equation}
    (A_{\Omega^L\to X}\mu)(x) = \int_{\Omega^L}\varphi^L(x,w)d\mu(w)
\end{equation}
for all $x\in X$ and $\mu\in \M(\Omega^L)$. By construction of $\B^L$, 
\begin{equation}
    f(x) = \braket{\tilde{\varphi}^L_{\varphi^{L-1}_x}}{h} = h(\varphi^{L-1}_x) = \int_{\Omega^{L-1}}\tilde{\varphi}^L(\varphi^{L-1}_x,w)d\mu(w) = (A_{\Omega^L\to X}\mu)(x)
\end{equation}
Hence, functions in $\B^L$ satisfy \eqref{eq:int_type-primal}. 

To check \eqref{eq:int_type-dual}, define 
\begin{equation}
    \begin{split}
    (A_{X\to\Omega^L}\pi)(w) &= \int_{X}\varphi^L(x,w)d\pi(x) \\
    (A_{\B^{(L-1)\di}\to \Omega^L}\rho)(w) &= \int_{\B^{(L-1)\di}}\tilde{\varphi}^L(g,w)d\rho(g) \\
    \iota(x) &= \varphi^{L-1}_x 
\end{split}    
\end{equation}
for all $w\in \Omega^L$, $x\in X$, $\pi\in \M(X)$ and $\rho\in \M(\B^{(L-1)\di})$. From the inclusion relation of $\B^{L\di}$ into $\tilde{\B}^{L\di}$ it follows that there exists a $\rho\in \M(\B^{(L-1)\di})$ for all $g\in\B^{L\di}$ such that $g=A_{\B^{(L-1)\di}\to \Omega^L}\rho$, but need to show that there exists a $\pi\in\M(X)$ such that $g=A_{X\to\Omega^L}\pi$. We use the definition of $\B^{L\di}$ and the map $\iota$ to explicitly construct a $\pi\in\M(X)$ for each $g\in \B^{L\di}$.

If $h\in \N(A_{\Omega^L\to X})$, then
\begin{equation}
    h(\iota(x)) = h(\varphi^{L-1}_x) = 0
\end{equation}
for all $x\in X$. Hence, if $g\in \B^L$, then
\begin{equation}
    0 = \braket{A^\di g}{h} = \braket{A_{\B^{(L-1)\di}\to \Omega^L}\rho}{h} 
    = \int_{\B^{(L-1)\di}}h(g)d\rho(g) = \int_{\B^{(L-1)\di}\backslash \iota(X)}h(g)d\rho(g)
\end{equation}
for all $h\in \N(A_{\Omega^L\to X})$ and some $\rho\in \M(\tilde{\B}^{(L-1)\di})$ such that $g=A^\di g=A_{\B^{(L-1)\di}\to \Omega^L}\rho$. Thus, any $g\in \B^{L\di}$ can be written as
\begin{equation}
    g=A_{\B^{(L-1)\di}\to \Omega^L}\rho = A_{X\to \Omega^L}\iota_{\#}\mathfrak{r}_X\rho
\end{equation}
for some $\rho\in \M(\B^{(L-1)\di})$ with restriction to $\iota(X)$ given by $\mathfrak{r}_X\rho\in \M(\iota(X))$. Since $\pi=\iota_{\#}\mathfrak{r}_X\rho$, the push-forward of the measure $\mathfrak{r}_X\rho$ from $\M(\iota(X))$ to $\M(X)$ along $\iota$, is an element of $\M(X)$, the functions in $\B^{L\di}$ satisfy \eqref{eq:int_type-dual}.

Last, observe that \eqref{eq:int_type-pairing} follows from
\begin{equation}
    \braket{g}{f} = \braket{A^\di g}{h} = \braket{A^\di g}{\mu} = \braket{A_{X\to \Omega^L}\pi}{\mu} = \int_{\Omega^L}\int_X\varphi^L(x,w)d\pi(x)d\mu(w)
\end{equation}
\end{proof}

This shows that integral cRKBS are indeed standard integral RKBS with more complicated kernels. In the next section, we will choose particular kernels which correspond to neural networks.

\subsection{Neural cRKBS}
Deep neural networks, like multi-layer perceptrons, are characterised by alternating affine transformations and elementwise nonlinearities. The natural infinite width extension can be described by an integral cRKBS with a certain kernel, which is composed of alternating affine transformations and elementwise nonlinearities. We call these spaces \textit{neural cRKBS}. In this section, we will describe how neural cRKBS are constructed and show that the neural tree spaces $\set{\W^\ell}_{\ell=1}^L$ defined in \citet{e_banach_2020} are a specific instance of such a neural cRKBS. 

We start with the neural network case with a single layer. First, we need to define a vector space structure on $X$ and $\Omega$, so we can meaningfully define the notion of a linear layer. We take $X\subseteq \V, \Omega\subseteq\V^{\di}\times \R$, where $V, V^{\di}$ are a dual pair of normed vector spaces. Note that in most applications, these spaces are $\R^d$, where $d$ is the dimension of the input data. The $\times\R$ is added to represent the bias. To be a neural RKBS, the kernel has to take the form 
\begin{equation}
\varphi(x,w)=\sigma(\braket{v}{x} + b)\beta(w)    
\end{equation}
where $w=(v,b)$, $\sigma:\R\to\R$ is some (nonlinear) function, called the activation function, and $\beta:\Omega\to\R$, a weighting function such that $\varphi$ is bounded. 

A well-known example of a neural RKBS is the Barron space, although it is more often written in the form where $\beta$ is moved into the total variation norm of the representing measures. For a detailed investigation of these spaces, see \citep{spek_duality_2023}.

\begin{definition}\label{def:neural_rkbs}
Let $X\subseteq \V, \Omega\subseteq\V^\di\times \R$, where $\V, \V^{\di}$ are a dual pair of normed vector spaces. Let $\sigma:\R\to \R$ and $\beta:\Omega \to \R$ be measurable and positive, and such that $\varphi:X\times \Omega\to \R$ is bounded, where $\varphi(x,w):= \sigma(\braket{v}{x} + b)\beta(w)$ for all $x\in X$ and $w=(v,b)\in \Omega$.

If $\B, \B^{\di}$ are an integral RKBS pair with kernel $\varphi$, then $\B, \B^{\di}$ are a neural RKBS pair.
\end{definition}

In this definition, the purpose of the weighting function $\beta$ is to ensure that the kernel $\varphi$ is bounded even when $\sigma$ and $\Omega$ are not. This is a technical complication to accommodate certain commonly used activation functions. For example, if $\B$ is a neural RKBS with domain $X$, some bounded subset of $\R^d$, and parameter space $\Omega=\R^d\times \R$, activation function $\sigma=\text{SoftPlus}:=x\mapsto \log(1+\exp(x))$ and weighting function $\beta(w) = (1+\norm{v}+\abs{b})^{-1}$. The activation function $\sigma$ is an unbounded Lipschitz continuous function with Lipschitz constant 1, and the weighting function ensures that
\begin{equation}
    \abs{\varphi(x,(v,b))} \leq \frac{\abs{\braket{v}{x}}+\abs{\sigma(0)}+\abs{b}}{1+\norm{v}+\abs{b}} \leq \frac{\norm{x}\norm{v}+\log(2)+\abs{b}}{1+\norm{v}+\abs{b}} \leq \max\set{\norm{x},1}
\end{equation}
for all $x\in X,(v,b) \in \Omega$. It is also possible to avoid using the weighting function $\beta$ by changing the definition of the integral RKBS to only consider measures $\mu$ such that \eqref{eq:int_type-primal} is integrable for all $x\in X$. However, this makes the analysis below much more complicated, so we opt to use a weighting function $\beta$ instead. 

To extend these ideas to chain RKBSs, we need a similar linear structure for the domains $\B^{(\ell-1)\di}, \Omega^\ell$ of the link spaces $\tilde{\B}^{\ell},\tilde{\B}^{\ell\di}$. We can do this by taking $\Omega^\ell \subseteq \B^{\ell-1}\times \R$, where we leverage the natural linear structure given by the pairing. Note that the $\di$-relation is cross-linked: $\tilde{\B}^{\ell}$ contains functions of $\B^{(\ell-1)\di}$ and $\tilde{\B}^{\ell\di}$ functions of $\B^{\ell-1}\times \R$. These link spaces are neural RKBS when their kernel is of the form
\begin{equation}\label{eq:link_neural_kernel}
    \tilde{\varphi}(g,(f,b)) = \sigma(\braket{g}{f}+b)\beta(f,b)
\end{equation}
for all $g\in \B^{(\ell-1)\di}$, $(f,b)\in \Omega^\ell\subseteq \B^{\ell-1}\times \R$, activation function $\sigma:\R\to\R$, and weighting function $\beta:\Omega^\ell\to\R$.

This leads to the formal definition of a neural cRKBS: a chain RKBS where the initial and link RKBS are neural RKBS. For a graphical description of the neural RKBS chain process, see Figure~\ref{fig:neural_chain}.

\begin{figure}
    \centering
    \includegraphics[width=0.9\linewidth]{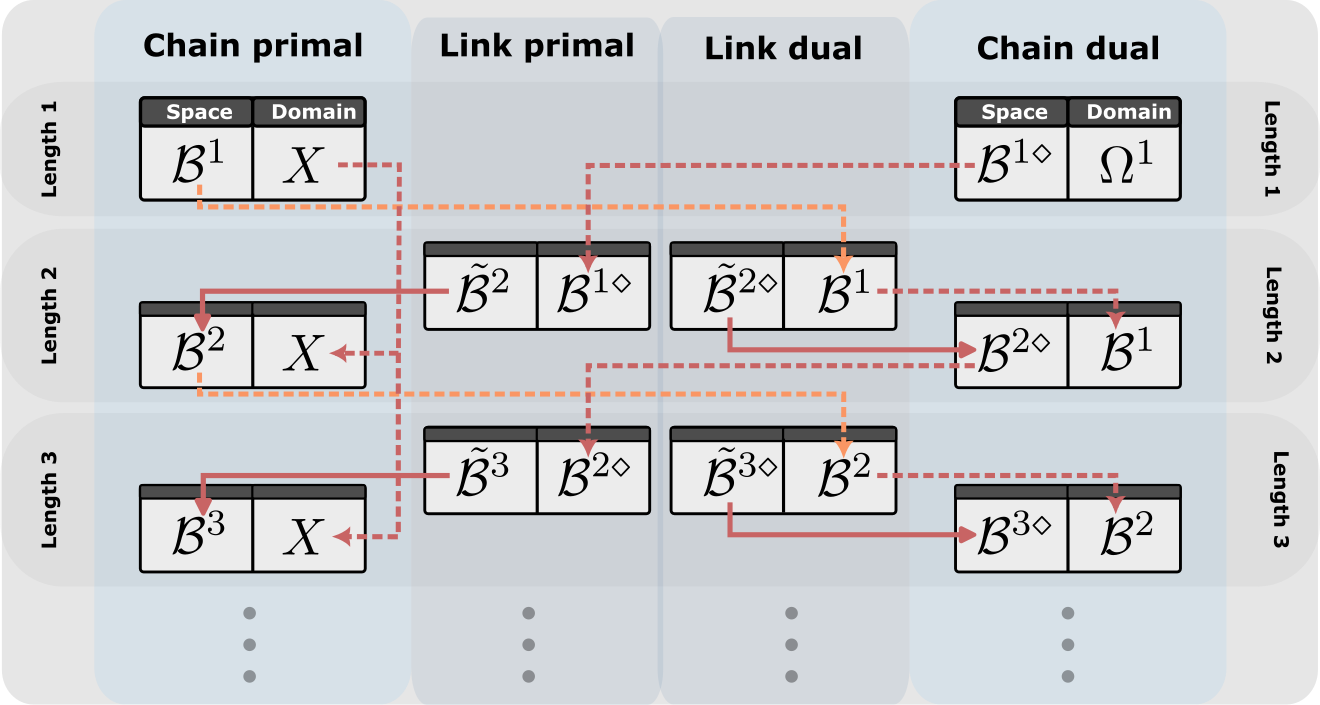}
    \caption{Schematic representation of the construction of a neural RKBS chain, with biases omitted from the diagram. The dependencies of the spaces have been indicated by arrows: a straight line when there is a map between the chain and link, and a dashed line for a domain identification. The orange arrows are the extra dependencies compared to the general case depicted in Figure~\ref{fig:kernel_chain}.}
    \label{fig:neural_chain}
\end{figure}

\begin{definition}
An integral cRKBS pair $\B^L,\B^{L\di}$ with domains $X,\Omega^L$ and kernel $\varphi^L$ are called a neural cRKBS pair if the initial RKBS pair $\B^1,\B^{1\di}$ is a neural RKBS pair with domains $X\subseteq \V,\Omega^1\subseteq\V\times \R$, with activation function $\sigma^1:\R\to\R$ and with weighting function $\beta^1:\Omega^1\to\R$, and if the link RKBS pairs $\tilde{\B}^\ell,\tilde{\B}^{\ell\di}$ are neural RKBS pairs with domains $\B^{(\ell-1)\di},\Omega^\ell\subseteq\B^{\ell-1}\times\R$, with activation functions $\sigma^\ell:\R\to\R$ and with weighting functions $\beta^\ell:\Omega^1\to\R$ for all $\ell=2,\hdots L$.
\end{definition}

Note that by construction these neural cRKBSs are 'regular' RKBSs: They have a proper norm, are complete vector spaces and have a kernel. This abstract definition implies that the functions of these RKBS have a particular recursive form:
\begin{equation}\label{eq:neural_fg}
\begin{split}
    f^L(x)&= \int_{\Omega^L}\sigma^L(f^{L-1}(x)+b)\beta^L(f^{L-1},b)d\mu(f^{L-1},b)\\
    g^L(w) &= \int_X \sigma^L(f^{L-1}(x)+b)\beta^L(f^{L-1},b) d\rho(x)
\end{split}
\end{equation}
where $x\in X, w=(f^{L-1},b)\in \Omega^L\subseteq \B^{L-1}\times \R$. To paraphrase this equation, a neural cRKBS of length $L$ has functions that are linear combinations of functions of a neural cRKBS of length $L-1$ with $\sigma$ applied to them. This form is due to the particular structure of the kernel $\varphi^L$.

\begin{theorem}\label{thm:neural_cRKBS_kernel}
If $\B^L,\B^{L\di}$ be a neural cRKBS pair with given $X,\Omega^\ell$, activation functions $\sigma^\ell$ and weighting functions $\beta^\ell$ for $\ell=1,\hdots, L$, then the kernel $\varphi^L:X\times \Omega^L$ is given by
\begin{equation}\label{eq:chain_neural_kernel}
    \varphi^L(x,(f^{L-1},b)) = \sigma^L(f^{L-1}(x)+b)\beta^L(f^{L-1},b)
\end{equation}
for all $x\in X$ and $(f^{L-1},b)\in \Omega$.
\end{theorem}
\begin{proof}
Let $\B^L,\B^{L\di}$ be a neural cRKBS pair with given $X,\Omega^\ell$, activation functions $\sigma^\ell$ and weighting functions $\beta^\ell$ for $\ell=1,\cdots, L$. By Theorem~\ref{thm:int_type_chain} the kernel $\varphi^\ell$ must be of the form $\varphi^L(x,w)=\tilde{\varphi}^L(\varphi^{L-1}(x,\cdot), w)$ for all $x\in X,w\in \Omega^L$. We can write $w=(f^{L-1},b)\in \Omega^L\subseteq \B^{L-1}\times \R$ and by \eqref{eq:link_neural_kernel} and the kernel property of $\varphi^{L-1}$ \eqref{eq:kernel} that
\begin{equation}
    \varphi^L(x,(f^{L-1},b)) = \tilde{\varphi}^L(\varphi^{L-1}(x,\cdot), w) = \sigma^L(\braket{\varphi^{L-1}(x,\cdot)}{f^{L-1}}+b)\beta^L(f^{L-1},b) = \sigma^L(f^{L-1}(x)+b)\beta^L(f^{L-1},b)
\end{equation}
for all $x\in X$ and $(f^{L-1},b)\in \Omega$. 
\end{proof}

As $f^{L-1}$ is an element of an integral RKBS, there exists a measure $\mu\in\M(\Omega^{L-1})$ such that $f^{L-1}=A_{\Omega^{L-1}\to X} \mu$. This allows us to expand the kernel $\varphi^L$ in the following way
\begin{equation}
    \varphi(x,(A_{\Omega^{L-1}\to X} \mu,b)) = \sigma^L\left(b+\int_{\Omega^{L-1}}\varphi^{L-1}(x,w)d\mu(w)\right) \beta^L(A_{\Omega^{L-1}\to X} \mu,b)
\end{equation}
Here we now more clearly see the neural network structure of the neural cRKBS: The kernel of layer $L$ is the composition of an activation function $\sigma^L$, an affine transformation with weights $\mu$ and bias $b$, and the kernel of the previous layer $L-1$. The non-linearity is elementwise because $\sigma$ acts on the function evaluated at $x$, which is the natural extension of the concept of elementwise to RKBSs. 

The neural cRKBS $\B^L$ with $\Omega^\ell=\B^{(\ell-1)}\times \R$ contains all the multi-layer perceptrons of length L with arbitrary widths at each layer with given activation functions. The key idea here is that we can choose $\mu$ as a linear combination of Dirac measures. When $\Omega^\ell$ are chosen smaller, $\B^L$ can be restricted to only networks where weights are bounded, or to represent convolutional layers, or other linear layers that respect some symmetry for example.

\begin{theorem}\label{thm:deep_networks_in_neural_cRKBS}
Let $L,d^0\in\mathbb{N}$. Let $B^L,\B^{L\di}$ be a neural cRKBS pair with $X\subseteq\R^{d^0}, \Omega^1=\R^{d^0}\times \R$, $\Omega^\ell=\B^{\ell-1}\times \R$ for $\ell=2,\hdots,L$, and activation functions $\sigma^\ell$ and weighting functions $\beta^\ell$, for $\ell=1,\hdots,L$. Furthermore, let $d^\ell\in \mathbb{N}$ and weight matrices $W^\ell\in \R^{d^{\ell+1}\times d^{\ell}}$ for $\ell=0,\hdots, L-1$ and bias vectors $b^\ell \in \R^{d^\ell}$ for $\ell=1,\hdots, L$ and row vector $W^L\in \R^{1\times d^L}$. If the deep neural network $f:X\to\R$ is given by
\begin{equation}
    f(x) = W^{L}\hat{\sigma}^L(b^L+W^{L-1}\hat{\sigma}^{L-1}(\cdots \hat{\sigma}^1(b^1 + W^0x) \cdots ))
\end{equation}
for all $x\in X$ in which $\hat{\sigma}^\ell:\R^{d^\ell}\to \R^{d^\ell}$ are the elementwise extensions of $\sigma^\ell$ for $\ell=1,\hdots, L$, then $f\in \B^L$ and
\begin{equation}
    \norm{f^L}_{\B^L}\leq \sum_{j=1}^{d^L} \frac{|W^{L}_{1j}|}{\beta^L\left(f_j^{L-1},b_j^L\right)}
\end{equation}
where
\begin{equation}
    f_j^{L-1}(x) = W_{j\cdot}^{L-1}\hat{\sigma}^{L-1}\left(\cdots \hat{\sigma}^1\left(b^1 + W^0x\right) \cdots \right)
\end{equation}
with $W^{L-1}_{j\cdot}$ the $j$th row of $W^{L-1}$.
\end{theorem}
\begin{proof}
We will prove that $f^L\in \B^L$ by induction on $L$. First, the case where $L=1$. Let $f^1$ be a shallow neural network
\begin{equation}
    f^1(x) = \sum_{j=1}^{d^1} W^1_{1j}\sigma\left(b^1_j + \sum_{n=1}^{d^0}W^0_{jn}x_n\right)
\end{equation}
We choose
\begin{equation}
    \mu:= \sum_{j=1}^{d^1} \frac{W^1_{1j}}{\beta^1(W^0_{j\cdot},b_j^1)} \delta_{\left(W^0_{j\cdot},b_j^1\right)}
\end{equation}
where $\delta$ are Dirac measures and $W^0_{j\cdot}$ the $j$th row of $W^0$. We have that $\mu\in\M(\Omega^L)$, as $\left(W^0_{j\cdot},b_j^1\right)\in \Omega^1=\R^{d^0}\times \R$. 
\begin{equation}
    f^1(x) = \sum_{j=1}^{d^1} \frac{W^1_{1j}}{\beta^1(W^0_{j\cdot},b_j^1)}\sigma\left(b^1_j + \sum_{n=1}^{d^0}W^0_{jn}x_n\right)\beta^1(W^0_{j\cdot},b_i^1) = \int_{\Omega^1} \sigma^1(\braket{v}{x}+b)\beta^1(v,b)d\mu(v,b)
\end{equation}
Hence, by Definitions~\ref{def:int_type}~and~\ref{def:neural_rkbs}, we get that $f^1\in \B^1$. 

Suppose $f^L$ is a neural network of length $L$. Then we can write $f^L$ as the last layer acting on a collection of $f^{L-1}_i$, for $i=1,\cdots,d^L$. 
\begin{equation}
    f^L(x) = \sum_{j=1}^{d^L} W^{L}_{1j} \sigma(b^L_j + f_j^{L-1}(x))
\end{equation}
By the induction hypothesis $f_j^{L-1}\in \B^{L-1}$ for all $j=1,\cdots, d^L$, so we choose 
\begin{equation}
    \mu:= \sum_{j=1}^{d^L} \frac{W^{L}_{1j}}{\beta^L(f_j^{L-1},b_j^L)}\delta_{\left(f_j^{L-1},b_j^L\right)}
\end{equation}
We have that $\mu\in\M(\Omega^L)$, as $(f_i^{L},b_i^L)\in \Omega^L=\B^{L-1}\times \R$. 
\begin{equation}
\begin{split}
    f^L(x) &= \sum_{j=1}^{d^L} \frac{W^{L}_{1j}}{\beta^L\left(f^{L}_j,b_j^L\right)}\sigma^L\left(b^L_i + f^{L-1}(x)\right)\beta^L\left(f^{L-1}_j,b_j^L\right) \\&= \int_{\Omega^1} \sigma^L(f^{L-1}(x)+b)\beta^L(f^{L-1},b)d\mu(f^{L-1},b)
\end{split}    
\end{equation}
Hence, by \eqref{eq:neural_fg} and Theorems~\ref{thm:int_type_chain}~and~\ref{thm:neural_cRKBS_kernel} we get that $f^L\in\B^L$.

Furthermore, the definition of the norm of $\B^L$ implies that 
\begin{equation}
    \norm{f^L}_{\B^L}\leq |\mu|(\Omega^L) = \sum_{j=1}^{d^L} \frac{|W^{L}_{1j}|}{\beta^L\left(f_j^{L-1},b_j^L\right)}
\end{equation}
\end{proof}

For most activation functions we can set $\beta$ to be constant, so the norm of $f^L$ is just the 1-norm of the weights of the last linear layer. 

\subsection{Relation to other spaces}
In this section, we will discuss how the neural cRKBS relates to other function spaces for deep networks, like generalised Barron spaces and hierarchical spaces. We will not consider bottlenecked spaces, as this approach based on function composition is incompatible with kernel composition.

The neural tree spaces of \citet{e_banach_2020} are a special case of a neural cRKBS $\B^L$. With the choices $X=K\subseteq \R^{d^0}$, some fixed compact set $K$, $\Omega^1 = \R^d\times \R$, $\Omega^\ell = \set{f\in \B^{\ell-1} \given \|f\|_{\B^{\ell-1}}\leq 1}\times \{0\}$, the unit ball in $\B^{\ell-1}$ without biases, $\sigma^\ell = ReLU$, the rectified linear unit and $\beta^\ell = 1$, because $\sigma$ is continuous and both $X$ and $\Omega^l$ are bounded, we exactly retrieve the definition of the space $\mathcal{W}^L(K)$ of Section 3.1 of \citet{e_banach_2020}. This paper however does not have a dual framework, so the $\B^{L\di}$ has no analogue. 

The authors show that for this space there are direct and inverse approximation theorems and that the norm in this space for deep networks is equivalent to a path norm
\begin{equation}
    \norm{f}_{\B^L} \leq \sum_{j_L=1}^{d^L} \cdots \sum_{j_1=1}^{d^1} \sum_{j_0=1}^{d^0} \left|W_{1j_L}^{L}\hdots W_{j_2 j_1}^{1} W_{j_1 j_0}^0\right|
\end{equation}
This result is to be expected as the ReLU functions are $1-$homogeneous. Therefore, the sizes of the inner weights can be moved to the last layer weight $W^L$ and keep the inner weights in a unit ball. One can reformulate this space with a full domain $\Omega^\ell=\B^{\ell-1}\times \R$ by introducing a weighting function $\beta^\ell(f^{\ell-1},b)=(\|f\|_{\B^{\ell-1}}+|b|)^{-1}$.

The original definition of the $\mathcal{W}^L(K)$ does not include a bias. In general, if the constant function $c:x\mapsto 1$ is included in $\B^{\ell-1}$, the term $f(x)+b$ with $f\in \B^\ell$ and $b\in\R$ can be rewritten as $(f+bc)(x)$. A push-forward argument shows that this leads to an isomorphism between the primals of the pairs with and without biases. It depends on the weighting function whether this isomorphism is isometric. Removing the bias implies a different $\Omega^\ell$, so the functions $g\in \B^{L\di}$ have a different domain. However, each of these $g$ can be thought of as a restriction to the domain without a bias.
\begin{lemma}
Let $\B^{L-1}, \B^{(L-1)\di}$ be a neural cRKBS pair with domains $X, \Omega^{L-1}$. Consider the neural link RKBS pairs $\tilde{\B}_1^{L},\tilde{\B}_1^{L\di}$ and $\tilde{\B}_2^{L},\tilde{\B}_2^{L\di}$ with domains $\B^{(L-1)\di},\Omega_1^L:=\B^{L-1}$ and $\B^{(L-1)\di},\Omega_2^L:=\B^{L-1}\times \R$ respectively, activation functions $\sigma^L_1=\sigma^L_2=\sigma$ as well as weighting functions $\beta_1,\beta_2$ satisfying $\beta_1(f)=\beta_2(f,0)$. Moreover, let $\B_i^L$ be the neural cRKBS constructed from $\B^{L-1},\B^{(L-1)\di}$ by linking $\tilde{\B}_i^{L},\tilde{\B}_i^{L\di}$ for $i=1,2$. If $\B^{L-1}\ni c:x\mapsto 1$ and $\sup_{f,b\in \B^{L-1}\times \R}\abs{\frac{\beta_2(f,b)}{\beta_1(f+bc)}}<\infty$, then $\B_1^L$ and $\B_2^L$ are isomorphic, and the restriction of $\B_2^{L\di}$ to $\Omega_1^L$ is $\B_1^{L\di}$.
\end{lemma}
\begin{proof}
The embedding of $\B_1^L$ into $\B_2^L$ is immediate. 

For the converse, let $f^L\in \B^L_2$. There exists a measure $\mu\in \M(\B^{L-1}\times\R)$ so that
\begin{align*}
    f^L(x)  
    &= \int_{\B^{L-1}\times \R}\sigma(f(x)+b)\beta_2(f,b)d\mu(f,b) \\
    &= \int_{\B^{L-1}\times \R}\sigma(f(x)+bc(x))\beta_2(f,b)d\mu(f,b) \\
    &= \int_{\B^{L-1}\times \R}\sigma((f+bc)(x))\beta_1(f+bc)\frac{\beta_2(f,b)}{\beta_1(f+bc)}d\mu(f,b) \\
    &= \int_{\B^{L-1}}\sigma(f(x))\beta_1(f)d\gamma(f)
\end{align*}
where $\gamma:=\Theta^{\#}\nu$ with
\begin{equation}
    \begin{split}
    \frac{d\nu}{d\mu}(f,b) &:= \frac{\beta_2(f,b)}{\beta_1(f+bc)} \\
    \Theta&: \B^{L-1}\times \R \to \B^L,\; (f,b)\mapsto (f+bc)
\end{split}    
\end{equation}
and
\begin{align}
    \norm{f} \leq \norm{\gamma} \leq \int_{\B^{L-1}\times \R}\abs{\frac{\beta_2(f,b)}{\beta_1(f+bc)}}d\abs{\mu}(f,b) \leq \sup_{f,b\in \B^{L-1}\times \R}\abs{\frac{\beta_2(f,b)}{\beta_1(f+bc)}} \norm{\mu}
\end{align}
Taking the infimum over $\mu$ shows the embedding of $\B^L_2$ into $\B^L_1$.

Denote with $\iota_{\text{res}}$ the restriction map, i.e.
\begin{equation}
     \iota_{\text{res}}g_2(f) = g_2(f,0)
\end{equation}
for $g_2\in\B^{L\di}_2$ and $f\in \B^{L-1}$. This map is linear with
\begin{equation}
    \norm{\iota_{\text{res}}} \leq 1
\end{equation}
and $\iota_{\text{res}}(\B_2^{L\di})=\B_1^{L\di}$
\end{proof}

Now we can show that the neural tree space $\W^L$ is equivalent to neural cRKBS $\B^L$ with $\Omega^\ell=\B^{\ell-1}\times \R$. 

\begin{theorem}
Let $L\in\mathbb{N}$, $X\subseteq \R^d$ compact, and denote with $\W^\ell$ the neural tree spaces. If $\sigma^\ell:=\relu:=x\mapsto \max(0,x)$, $\beta^\ell(f,b)=\frac{1}{\norm{f}+\abs{b}}$ and $\Omega=\R^{d+1}$, then $\B^L,\B^{L\di}$ is a neural cRKBS with $\B^\ell$ isometrically isomorphic to $\W^\ell$ for all $\ell=1,\hdots,L$.   
\end{theorem}
\begin{proof}
This is a proof by induction, with the base case following from \citet{spek_duality_2023}. Moreover, $\norm{c}\leq 1$, which follows by choosing the measure $\delta_{0,1}$.

For the induction step, assume that $c\in \B^{\ell}$ with $\norm{c}\leq 1$ and $W^\ell$ isometrically isomorphic to $\B^\ell$.
Define
\begin{equation}
    \begin{split}
    A_{\W}\mu(x) :&= \int_{\W^\ell}\sigma(f^\ell(x))\beta^\ell(f^\ell,0)d\mu(f^\ell) \\
    \tilde{\W}^{\ell+1} :&= \set{ f:X\to\R \given f = A_{\W}\mu} \\
    \norm{f}_{\tilde{\W}^{\ell+1}} :&= \inf_{f=A_{\W}\mu}\norm{\mu}
\end{split}
\end{equation}
To show the isometric isomorphism between $\W^{\ell+1}$ and $\B^{\ell+1}$, we will show that both are isometrically isomorphic to $\tilde{\W}^\ell$. The isometric isomorphism between $\W^{\ell+1}$ and $\tilde{\W}^{\ell+1}$ follows from the homogeneity of ReLU. For the isometric isomorphism between $\tilde{\W}^{\ell+1}$ and $\B^{\ell+1}$, observe that the choice of $\beta^{\ell+1}$ and the assumption on $c$ together imply that
\begin{equation}
    \sup_{f,b\in\B^{\ell}}\abs{\frac{\beta^\ell(f,b)}{\beta^\ell(f+bc,0)}} = \sup_{f,b\in\B^{\ell}}\frac{\norm{f+bc}}{\norm{f}+\abs{b}} \leq \max\set{1,\norm{c}} = 1
\end{equation}
It follows that the versions of $\B^{\ell+1}$ with and without bias are isometrically isomorphic. This combined with the isometric isomorphism between $\W^\ell$ and $\B^\ell$ shows the isometric isomorphism between $\B^{\ell+1}$ and $\tilde{\W}^{\ell+1}$. What remains to show is that $c\in \B^{\ell+1}$ with $\norm{c}\leq 1$. This follows by choosing the measure $\delta_c$.
\end{proof}

Neural Hilbert Ladders are defined by means of sequences of Hilbert spaces. These sequences have a similar pattern as the construction of the cRKBSs. In fact, the considered Hilbert spaces are neural cRKHS. 

A cRKHS is constructed using a RKHS $\H^1$ over $\X$ and a RKHS $\tilde{\H}^2$ over $\H^1$.  Following the procedure outlined for cRKBS in section \ref{sec:chaining_rkbs}, the feature $\psi^2:X\to \Tilde{H}^2$ is given by
\begin{equation}
    \psi^2(x) = \tilde{k}^2_{k^1_x}
\end{equation}
in which $k^1_\bullet$ is the kernel of $\H^1$ and $\tilde{k}^2_\bullet$ is the kernel of $\tilde{\H}^2$. The resulting kernel $k^2$ is 
\begin{equation}
    k^2(x_1,x_2) = \braket{\psi^2(x_1)}{\psi^2(x_2)}_{\Psi^2} = \braket{\tilde{k}^2_{k^1_{x_1}}}{\tilde{k}^2_{k^1_{x_2}}}_{\Psi^2}
\end{equation}
the kernel of the RKHS $\H^2$, in which the Hilbert space $\Psi^2$ is a feature space for $\H^2$.

If the cRKHS is of neural-type with $\Psi^2=L^2(\H^1, \pi)$ for some probability measure with bounded second moment $\pi$ and the activation function $\sigma$ is Lipschitz, then the kernel $k^2$ satisfies 
\begin{equation}
    k^2(x_1,x_2) = \int_{\H^1}\sigma(g(x_1))\sigma(g(x_2))d\pi(g)
\end{equation}
This agrees with the structure of $\H_{k^2}$ in \eqref{eq:NHL_kernel}. Repeating this process shows that the Hilbert spaces in the infimum of \eqref{eq:NHL_quasi-norm} are indeed neural cRKHSs. However, the space $\F^{(L)}_p$ constructed by taking the infimum over these cRKHSs, loses the RKHS structure, in contrast to the neural cRKBS.  

\section{Representer Theorem: kernel chains enable weight sharing}
So far, we have shown that neural cRKBSs are a natural infinite width limit of deep neural networks, which gives insight into their general properties. In most applications, however, we have only a finite amount of data points. In this case, all functions in such neural cRKBSs can be represented by (finite) deep neural networks. 

Given $N$ data points, an application of the representer theorem to \eqref{eq:neural_fg}, reduces the integral of $f^L\in\B^L$ to a linear combination of $N$ functions $f^{L-1}\in\B^{L-1}$. Repeating this procedure would lead to an exponential amount of nodes in $N$. Here the kernel and duality structure of our spaces can significantly improve this estimate, as a set of at most $N$ evaluation functionals form a basis of an RKBS pair, when we have only finite data points. Decomposing the $f^{L-1}$ using a basis, allows them to share weights and leads to a neural network with at most $N$ hidden nodes at each layer.

\begin{lemma}\label{lemma:shared_basis}
Let $\B, \B^\di$ be a pair of RKBS with domains $X,\Omega$, and kernel $K$. If $|X|$ is finite, then $\B^{\di}=\spn\set{K_x\given x\in X}$ and there exists $N$ elements $w_j\in \Omega$ such that $\set{K_{w_j}\given j=1,\hdots N}$ forms a basis of $\B$, where $N=\dim \B^{\di}$.
\end{lemma}
\begin{proof}
Let $X$ have finite cardinality. Choose $\set{x_1, \hdots, x_N}\subseteq X$ such that the set $\set{K_{x_i}\given i=1,\hdots, N}$ is a maximally linear independent subset of $\set{K_x\given x\in X}$.

We can pick $w_j\in \Omega$, for $j=1,\hdots N$, such that the matrix $\hat{K}$ with the elements $\hat{K}_{ij}=K(x_i,w_j)$ has full rank, because it has $N$ linear independent rows $K(x_i,\cdot) = K_{x_i}$ for $i=1,\hdots, N$. Thus, the columns of this matrix $K(\cdot, w_j)=K_{w_j}\in \B$, for $j=1,\hdots, N$, must also be linearly independent. 

It remains to show that $\set{K_{w_j}\given j=1,\hdots N}$ spans $\B$. Let $f\in \B$, and take $c=\hat{K}^{-1} (f(x_1),\hdots, f(x_N))^T\in \R^{N}$. Define $f_0:=f-\sum_{j=1}^N c_j K_{w_j}$, then $f_0(x_i)= f(x_i)- \sum_{j=1}^N c_j K(x_i,w_j)=0$ for all $i=1,\hdots, N$. With the set $\set{K_{x_i}\given i=1,\hdots, N}$ chosen to be maximally linearly independent $f_0(x)=0$ for all $x\in X$, the definition of a Banach space of functions implies that $f_0$ is the zero vector and $f=\sum_{j=1}^N c_j K_{w_j}$. 

Since $\B,\B^{\di}$ are a dual pair of Banach Spaces, they embed in each other duals: 
\begin{equation}
    \B \hookrightarrow \B^{\di\ast} \hspace{2cm} \B^{\di} \hookrightarrow \B^\ast
\end{equation}
We have shown previously that $\B$ has dimension $N$ which together with the embeddings implies that
\begin{equation}
    N = \dim(\B) = \dim(\B^\ast) \geq  \dim(\B^\di) = \dim(\B^{\di\ast}) \geq \dim(\B) = N
\end{equation}
Thus, $\B$ and $\B^\di$ have the same dimension, $N$. The set $\set{K_{x_i}\given i=1,\hdots, N}$ is linear independent and of size $N$, hence it also must span $\B^{\di}$ and the full set $\set{K_x\given x\in X}$ does as well.
\end{proof}

Using this Lemma, we can prove the main theorem of this section, which can be thought of as the converse of Theorem~\ref{thm:deep_networks_in_neural_cRKBS}: When $X$ is finite, all functions $f\in \B^L$ are deep networks. In this case, they even have at most $N$ hidden neurons at each layer and the inner weights and biases are fixed for a given $X$ and functions $\sigma^\ell, \beta^\ell$. 

\begin{theorem}\label{thm:finite_data_neural_cRKBS}
Let $L,d^0\in\mathbb{N}$. Let $B^L,\B^{L\di}$ be a neural cRKBS pair with $X\subseteq\R^{d^0}, \Omega^1\subseteq\R^{d^0}\times \R, \Omega^\ell\subseteq\B^{\ell-1}\times \R$, activation functions $\sigma^\ell$ and weighting functions $\beta^\ell$, where $\ell=2,\hdots,L$.

If $|X|$ is finite and $N=\dim \B^{L\di}$, then there exists weight matrices $W^0\in \R^{d^0\times N}, W^{\ell}\in \R^{N\times N}$, for $\ell=1,\hdots, L-1$, and bias vectors $b^\ell \in \R^{N}$ for $\ell=1,\hdots, L$ such that for all $f\in \B^L$, there exists a row vector $W^{L}\in \R^{1\times N}$ such that for all $x\in X$
\begin{align}
    f(x) &= W^{L}\hat{\sigma}^L\left(b^L+W^{L-1}\hat{\sigma}^{L-1}\left(\cdots \hat{\sigma}^1\left(b^1 + W^0x\right) \cdots \right)\right) \label{eq:f_as_neural_network}\\
    \|f\|_{\B^L} &= \sum_{j=1}^{N} \frac{|W^{L}_{1j}|}{\beta^L\left(f_j^{L-1},b_j^L\right)}
\end{align}
where $\hat{\sigma}^\ell:\R^{d^\ell}\to \R^{d^\ell}$ are the elementwise extensions of $\sigma^\ell$, for $\ell=1,\hdots L$, and where
\begin{equation}
    f_j^{L-1}(x) = W_{j\cdot}^{L-1}\hat{\sigma}^{L-1}\left(\cdots \hat{\sigma}^1\left(b^1 + W^0x\right) \cdots \right)
\end{equation}
\end{theorem}
\begin{proof}
We will prove \eqref{eq:f_as_neural_network} by induction on $L$. The base case follows directly from Lemma~\ref{lemma:shared_basis}, with kernel $\varphi(x,w) = \sigma^1(\braket{v}{x}+b)\beta(x,w)$.

Let $X$ have finite cardinality and let $N=\dim \B^{L\di}$. Then by Lemma~\ref{lemma:shared_basis}, there exists $\left(f_j^{L-1},b_j\right)\in \Omega^L$ for $j=1,\hdots, N$, such that $\{K_{\left(f_j^{L-1},b_j\right)}|j=1,\hdots,N\}$ are a basis of $\B^L$. We set $b^L:= (b_1,\hdots,b_n) \in \R^N$. Then by the induction hypothesis on $\B^{L-1}$ there exists weight matrices $W^0\in \R^{d^0\times N}, W^{\ell}\in \R^{N\times N}$, for $\ell=1,\hdots, L-2$, and bias vectors $b^\ell \in \R^{d^\ell}$ for $\ell=1,\hdots, L-1$, and row vectors $W_j^{L-1}\in \R^{1\times N}$ for $j=1,\hdots, N$ such that
\begin{equation}
    f_j^{L-1}(x) = W_j^{L-1}\hat{\sigma}^{L-1}\left(b^{L-1}+W^{L-2}\hat{\sigma}^{L-2}\left(\cdots \hat{\sigma}^1\left(b^1 + W^0x\right) \cdots \right)\right)
\end{equation}
We choose the weight matrix $W^{L-1}$ as the matrix with rows $W_j^{L-1}$ for $j=1,\hdots,N$.

We now only need to choose a $W^{L}$ dependent on $f^L$ to finish the proof. Let $f^L\in \B^L$. The set $\{K^L_{\left(f_j^{L-1},b_j\right)}|j=1,\hdots,N\}$ is a basis of $\B^L$. Hence, there exists scalars $c_j\in \R$, $j=1,\hdots, N$ such that
\begin{equation}
    f^L(x)= \sum_{j=1}^N c_jK^L_{\left(f_j^{L-1},b_j\right)}(x) = \sum_{j=1}^Nc_j\varphi^L\left(x,\left(f_j^{L-1},b_j\right)\right) = \sum_{j=1}^N c_j\beta^L\left(f_j^{L-1},b_j\right) \sigma^L\left(b_j+f_j^{L-1}(x)\right)
\end{equation}
where this last equality is due to Theorem~\ref{thm:neural_cRKBS_kernel}. Finally, set $W_{1j}^{L} := c_j\beta^L\left(f_j^{L-1},b_j\right)$, which completes the proof of \eqref{eq:f_as_neural_network}.

The norm of $f$ is by Definition~\ref{def:int_type} given by 
\begin{equation}
    \norm{f}_{\B^L} = \inf_{f=A_{\Omega^L\to X}\mu} |\mu|(\Omega^L)
\end{equation}
As $X$ is finite, we can see this as an optimisation problem with finite constraints. The Representer Theorem (cf \citep{bredies_sparsity_2019}) says that the infimum must be attained by a $\mu$ of the form
\begin{equation}
    \mu= \sum_{j=1}^{d^L} \frac{W^{L}_{1j}}{\beta^L(f_j^{L-1},b_j^L)}\delta_{\left(f_j^{L-1},b_j^L\right)}
\end{equation}
Hence,
\begin{equation}
    \norm{f}_{\B^L} = |\mu|(\Omega^L)= \sum_{j=1}^{N} \frac{|W^{L}_{1j}|}{\beta^L\left(f_j^{L-1},b_j^L\right)}
\end{equation}
\end{proof}

\begin{corollary}[Deep representer theorem for neural cRKBS]
Let $L,d^0\in\mathbb{N}$. Let $B^L,\B^{L\di}$ be a neural cRKBS pair with $X\subseteq\R^{d^0}, \Omega^1\subseteq\R^{d^0}\times \R, \Omega^\ell\subseteq\B^{\ell-1}\times \R$, activation functions $\sigma^\ell$ and weighting functions $\beta^\ell$, where $\ell=2,\hdots,L$.

Let $N\in \mathbb{N}$ and consider data points $(x_i,y_i)\in X\times \R$, for $i=1,\hdots, N$. If $E_y:\R^N\to\R$ is proper convex, coercive, lower semi-continuous with respect to the Euclidean norm, and dependent on $(y_1,\hdots,y_N)$, then the optimisation problem 
\begin{equation}\label{eq:unconstrained_optimisation}
    \inf_{f\in\B^L} E_y(f(x_1),\hdots,f(x_N)) + \|f\|_{\B^L}
\end{equation}
admits solutions of the form
\begin{equation}
    f(x) = W^{L}\hat{\sigma}^L\left(b^L+W^{L-1}\hat{\sigma}^{L-1}\left(\cdots \hat{\sigma}^1\left(b^1 + W^0x\right) \cdots \right)\right)  
\end{equation}
for $x\in X$. Here, the parameters are the weight matrices $W^0\in \R^{d^0\times N}, W^{\ell}\in \R^{N\times N}$, for $\ell=1,\hdots, L-1$ and the bias vectors $b^\ell \in \R^{N}$ for $\ell=1,\hdots, L$ and the row vector $W^{L}\in\R^{1\times N}$, and the $\hat{\sigma}^\ell:\R^{d^\ell}\to \R^{d^\ell}$ are the elementwise extensions of $\sigma^\ell$, for $\ell=1,\hdots L$. 
\end{corollary}
\begin{proof}
We can equivalently restate $\eqref{eq:unconstrained_optimisation}$ as an optimisation over $\mu\in\M(\Omega)$. Then by \citep[Theorem 3.3]{bredies_sparsity_2019}, there exists a solution a $\mu \in \M(\Omega^L)$ and the result follows immediately by Theorem~\ref{thm:finite_data_neural_cRKBS} for $f=A\mu\in \B^L$, where we can restrict $X$ to ${x_1,\hdots, x_N}$.  
\end{proof}

\begin{remark}
By construction, the neural cRKBS have only real-valued functions, and so can only contain neural networks with a single output dimension. However, when we have vector-valued output $y\in \R^D$, such a neural network can be thought of as a vector of $f^L_j\in \B^L$. As these scalar-valued deep networks share weights, we get a similar result as above but with $W^{L+1} \in \R^{D\times N}$. 
\end{remark}

When we have finite data, the norm on the dual space $\B^{\di}$ for an integral pair of RKBS gives a complexity measure on $\B$. Note that by replacing the $\xi$ in the theorem with Gaussian variables, we get the Gaussian complexity instead of the Rademacher complexity.

\begin{theorem}
Let $\B, \B^\di$ be an integral pair of RKBS with domains $X,\Omega$, and kernel $\varphi$. Let $|X|=N$ be finite, and let $\sigma_1,\hdots,\sigma_N$ be i.i.d. Rademacher random variables. Define the $B^{\di}$-valued random variable $G_\sigma$ as follows
\begin{equation}
    G_\xi= \sum_{i=1}^N \xi_i \varphi_{x_i}
\end{equation}
The empirical Rademacher complexity of the unit ball of $\B$ for $X$ is given by the expectation of the norm of $G_\sigma$
\begin{equation}
    \mathrm{Rad}_X(\set{f\in\B\given \|f\|_{\B}\leq 1}) = \frac{1}{N} \E_{G_\xi}\left(\|G_\xi\|_{B^\di}\right)
\end{equation}
\end{theorem}
\begin{proof}
As $\varphi_w = A_{\Omega \to X} \delta_w$ for $w\in \Omega$, $\|\varphi_w\|_{\B}=1$. By Lemma~\ref{lemma:shared_basis} the convex hull of the set $\set{\pm\varphi_w\given w\in \Omega}$ gives the unit ball of $\B$, so we may equivalently look at the Rademacher complexity of this set.
\begin{equation*}
\begin{split}
\text{Rad}_X(\set{f\in\B\given \|f\|_{\B}\leq 1}) &= \frac{1}{N}\E_{\xi} \left(\sup_{w\in \Omega} \left|\sum_{i=1}^N \xi_i \varphi(x_i, w) \right|\right)\\
&=\frac{1}{N}\E_{G_{\xi}} \left(\sup_{w\in \Omega} \left|G_{\sigma}(w)\right| \right)\\
&= \frac{1}{N}\E_{G_\xi}\left(\|G_{\xi}\|_{\B^{\di}} \right)
\end{split}
\end{equation*}
\end{proof}

\section{Discussion}
In this paper, we have shown that by composing kernels of RKBS pairs, we can form a natural function space for deep neural networks. These neural cRKBSs contain all deep neural networks. For finite data, the functions in these spaces reduce to deep networks of a width of at most the size of the data. 

Such neural cRKBS have a regular structure from an analysis standpoint: They have norms, are complete, and have a well-described duality structure between weights and data: On one side, $\B^L$ contains functions of the data $X$ and on the other side $\B^{L\di}$ contains functions over $\Omega^L$ which represents the set of all the weights, except for the final layer. Here, we also see, why the RKBS framework is more suitable to describe deep neural networks than an RKHS framework. In a chain RKBS the $X$ stays the same, while the $\Omega$ can adapt: for a deeper neural network, the input size stays the same, but we get more weights. For a chain RKHS, the symmetry requires that $\Omega$ stays fixed, equal to $X$. 

In this neural cRKBS space, the resulting optimisation problem \eqref{eq:unconstrained_optimisation} actually becomes convex. However, if we reduce to finite deep neural networks using the representer theorem, finding the inner weights $W^1,\hdots, W^{L-1}$ is a non-convex problem. The trade-off between finite dimensionality and convexity is a common feature in RKBSs, in comparison to RKHSs, where the representer theorem gives constructive solutions. This corresponds to deep neural networks, where gradient descent is used to overcome the non-convexity. 

In some sense, the neural cRKBS gives us deep neural networks where instead of finding a finite set of weights, we take any (Radon) distribution of weights. The final linear layer brings it back together by choosing a linear combination of all the possibilities, which is a convex problem. In previous work, \citep[Corollary 13]{spek_duality_2023}, we have shown that integral RKBSs are a union of RKHSs each corresponding to a choice of inner weights, showing that indeed these RKBSs form the set of all possible choices. 

The construction using kernel chains provides a concrete description of the dual space $\B^{L\di}$. Its norm is a measure of the expressivity of $\B$ via the complexity of its unit ball. The optimisation problem \eqref{eq:unconstrained_optimisation} for the neural cRKBS also allows for a dual problem \citep[Theorem 18 and 20]{spek_duality_2023}. However, this again leads to a non-convex problem, where we need to find the inner weights $w\in \Omega^L$ that form the basis of $\B^L$.

\section{Summary and outlook}
In this final section, we summarise our work at a higher level and highlight research areas where our theory could have a potential impact in the future.

In this work, we addressed the fundamental question of identifying an appropriate function space for deep networks. At first glance, one might think that function compositions are a natural concept for networks in depth. However, at second glance, they have the unique limitation in depth that extra hidden bottleneck layers (with potentially different widths) appear, which prevent desirable reproducing properties in the function space. To overcome this key limitation, we introduced the reproducing kernel chain framework (cRKBS). By composing kernels rather than functions, we could naturally preserve desired properties from single-layer reproducing kernel Banach spaces through depth.

This paradigm shift via kernel representations and duality has the potential to offer new perspectives and solutions for deep learning theory and practice in the future. In the following paragraphs, we highlight four of such potential impact areas.

In geometric deep learning \citep{bronstein_geometric_2021}, weight-sharing in widths has offered fascinating insights over the past years from translational-invariant CNNs via group-invariant layers to Clifford algebra layers in GNNs. However, especially given deep networks representing solutions via neural ODEs, neural PDEs, or more generally via operator learning \citep{boulle_mathematical_2023}, a promising prospect would be to study geometric and sparsity properties of cRKBS as \textit{weight-sharing in-depth}.

In the context of random feature models and optimisation, the role of \textit{sampling weights} of deep neural networks \citep{bolager_sampling_2023} has proven to be a successful concept for the construction of trained networks orders of magnitude faster than with conventional training schemes. Because of the primal-dual nature of the chains and links in our kernel chaining framework, our findings could offer new insights into deep weight sampling as well as for \textit{sparsity} promoting optimization in deep architecture search \citep{heeringa_sparsifying_2024,heeringa_learning_2023}.

An important prospect of this work also lies in \textit{generalisation theory} \citep{zhang_understanding_2021}. Building upon existing work on Rademacher complexity for function classes, it could be interesting to explore the role of reproducing kernel chains and bounds related to the VC dimension. For this, the finite-data to finite-weight relationship in our reproducing kernel chaining framework could be of particular interest.

Despite the recent big success of graph neural networks and transformers for imaging and large language models, a deeper understanding of the \textit{attention mechanism in depth} would be important. Whereas for basic graph neural networks, the connection to Barron spaces is investigated, a more detailed comparison of our work with this work \citep{wright_transformers_2021} could be promising to explore further.

Those four impact areas alone are already very promising. However, throughout this work, it was also an interesting observation that the reproducing kernel chaining concept is more general than deep neural networks. It is intriguing to think in the future about relaxing the assumptions on deep neural networks and what this could offer in practice.

\bibliographystyle{plainnat}
\bibliography{references}

\end{document}